\documentclass[12pt]{l4dc2023}

\usepackage{todonotes}

\title[Transportation-inequalities and Lyapunov stability]{Transportation-Inequalities, Lyapunov Stability and Sampling for Dynamical Systems on Continuous State Space}
\usepackage{times}

\author{\Name {Muhammad Abdullah Naeem} 
\Email{muhammad.abdullah.naeem@duke.edu} \\
\addr Department of Electrical and Computer Engineering\\
Duke University\\
Durham, NC 27708, USA
\AND
\Name {Miroslav Pajic} \Email{miroslav.pajic@duke.edu} \\
\addr Department of Electrical and Computer Engineering\\
Duke University\\
Durham, NC 27708, USA}

\begin{document}

\maketitle

\begin{abstract}%
We study the concentration phenomenon for discrete-time random dynamical systems 
with an unbounded state space. We develop a heuristic approach towards obtaining exponential concentration inequalities for dynamical systems using an entirely functional analytic framework. We also show that existence of exponential-type Lyapunov function, compared to the purely deterministic setting, not only implies stability but also exponential concentration inequalities for sampling from the stationary distribution, via \emph{transport-entropy inequality} (T-E). These results  
have significant impact in \emph{reinforcement learning} (RL) and \emph{controls}, 
leading to exponential concentration inequalities even for unbounded observables (i.e., rewards), while neither assuming reversibility nor exact knowledge of the considered random dynamical system (assumptions at heart of concentration inequalities in statistical mechanics and Markov diffusion processes).
\end{abstract}


\section{Introduction}
\label{submission}

\paragraph{Motivation.} Last decade has seen tremendous advancements in non-asymptotic analysis of system identification and optimal control for linear time-invariant (LTI) dynamical systems (e.g.,~\cite{tu2018least,hao2020provably, oymak2019stochastic,fazel2018global,simchowitz2018learning,sarkar2019finite,zahavy2019average}). Techniques and analysis developed in this paper have been initially motivated by  non-asymptotic analysis of average reward-based optimal control of linear dynamical system (LDS) and switched linear dynamical system (SLDS) where expected value of the reward w.r.t stationary distribution of a Markov chain is approximated, with high probability, by its empirical averages. Although sample complexity for control/dynamical systems on continuous state spaces has been extremely popular recently, we still lack fundamental understanding of the factors leading to sharp concentration inequalities. 

For example, every ergodic Markov chain is mixing (i.e., correlation decreases asymptotically), but as we will see in the following sections that does not necessarily imply exponential concentration (vaguely speaking: empirical averages concentrate sharply around expectation w.r.t to stationary measure ), especially from a single trajectory. To begin with, concentration with respect to (w.r.t) which class of functions? To expound upon these questions, in this work we  leverage upon \emph{concentration of measure phenomenon} for the process level law of a Markov chain via {transport-entropy inequality} (T-E) inequalities (for a good monograph on this topic see e.g.,~\cite{marton2004measure, djellout2004transportation}). The T-E approach is compact, concise and clear; even leading to exponential concentration for unbounded reward function of the form $r(x):= \|x\|$  for dynamical~systems. 

\paragraph{Related work.} Techniques used to prove concentration of a measure include martingale methods, exchangeable pairs (e.g., see~\cite{chatterjee2007stein}) and functional inequalities (transportation-entropy, logarithmic Sobolev/hypercontractivity and Poincare)~\cite{ledoux2001concentration}. However, exploring measure concentration for dependent random variables as in Markov chains on unbounded spaces limits the application of martingale methods and exchangeable pairs. The RL and system identification communities, pretty much exclusively, have used \emph{independent block technique} \cite{oymak2019stochastic, tu2018least} that seems to work when observables (i.e., rewards) are \emph{bounded}. SLDS fall into the category of Harris ergodic Markov chain and following relevant work on them in  \cite{latuszynski2013nonasymptotic}, by using martingale and independent block techniques, authors in \cite{naeem2020learning} were able to bound the mean-squared error between the empirical average and the expected reward (unbounded) from stationary distribution of SLDS; however, the employed probabilistic methods are tedious, opaque and lead to weak results. 

Stein methods (see e.g., \cite{chatterjee2010applications}), are well suited towards discrete setting in statistical physics but have limitations for the models under consideration. Only recently, \cite{wang2020transport} started a formal, functional analytic study for concentration inequalities in discrete-time setting using the transport-information~(T-I) inequality. Verification of T-I is plausible either in the discrete state space setting, or when a Markov chain is reversible and posseses a spectral gap in space of square integrable functions w.r.t its stationary distribution (see \cite{wang2020transport} for more details).
As we will see in concentration for nonlinear random dynamical systems, spectral gaps might only exist in Wasserstein sense. Reversibility assumption originates from study of Langevin-type stochastic differential equations used to model physical phenomenons in the nature; however, the scope of this paper is not limited to the reversibility~assumption.  

Note that~\cite{blower2005concentration,djellout2004transportation} have studied T-E inequalities for the case of stable LDS. 
However, a general framework to provide exponential concentration inequalities for deviation of empirical averages of a Markov chain with respect to (w.r.t) some unbounded test function is missing. Concentration for nonlinear random dynamical systems is not well explored, partly because they are  not necessarily contractive in trivial metric.
 However, leveraging on weighted transport-entropy inequalities introduced by \cite{bolley2005weighted} and exponentially fast convergence of Harris ergodic Markov chains in Wasserstein metric (\cite{hairer2011yet}), in this work, we still manage to provide sharp concentration by introducing an  exponential Lyapunov~condition.

\paragraph{Notation.} We use ${I}_{n}\in\mathbb{R}^{n\times n}$ to denote the $n$ dimensional identity matrix. 
For random variables $x$ and $y$, $Cov(x,y)$ denote the covariance. 
$ \mathcal{B}_{\alpha}^{n}:=\{x \in \mathbb{R}^{n}:  \|x\|:= \|x\|_2
\leq \alpha \}$ is the 
$\alpha$-ball in $\mathbb{R}^n$.   $\chi_{\{\}}()$ is the indicator function, whereas $\rho(A)$ and $\|A\|_2$ represent the spectral radius and the usual matrix 2-norm of $A \in \mathbb{R}^{n \times n}$, respectively. 
A sequence ${\{a(N)\}_{N \in \mathbb{N}} \in  \mathcal{O}(N)}$, if it increases at most linearly in $N$ (this is not limited to asymptotic results). 
Space of probability measure on  $\mathcal{X}$(continuous space) is denoted by  $\mathcal{P(\mathcal{X})}$ and space of its Borel subsets is represented by $\mathbb{B}\big(\mathcal{P(\mathcal{X})}\big)$. For a function $r$ and $\mu \in \mathcal{P(X)}$, we use $<r>_{\mu}$ to denote expectation of $r$ w.r.t $\mu$. Finally, for a set $\mathcal{K}\subseteq\{1,...,M\}$, its complement is $\mathcal{K}^\complement:=\{1,...,M\}\setminus\mathcal{K}$. 

On a metric space $(\mathcal{X},d)$, for $\mu, \nu \in \mathcal{P(\mathcal{X})}$, we define Wasserstein metric of order $p \in [1, \infty)$~as
\begin{equation}
\label{eq:WM}
    \mathcal{W}_{d}^{p} (\nu,\mu)= \bigg(\inf_{(X,Y) \in \Gamma(\nu,\mu)} \mathbb{E}~d^{p}(X,Y)\bigg)^{\frac{1}{p}};
\end{equation}
here, $\Gamma(\nu,\mu) \in P(\mathcal{X}^{2})$, and $(X,Y) \in \Gamma(\nu,\mu)$ implies that random variables $(X,Y)$ follow some probability distributions on $P(\mathcal{X}^{2})$ with marginals $\nu$ and $\mu$. Another way of comparing two probability distributions on $\mathcal{X}$ is via relative entropy, which is defined as
\begin{equation}
\label{eq:ent} 
    Ent(v||u)=\left\{ \begin{array}{lr}
    \int \log\bigg(\frac{d\nu}{d\mu}\bigg) d\nu, & \text{if}~ \nu << \mu,
         \\ +\infty, & \text{otherwise}. 
         \end{array}\right.
\end{equation}

\subsection{Problem Statement}
Under the action of some state dependent policy $\pi$, we consider a closed-loop random dynamical system of the form
\begin{equation}
\label{eq:crds}
    x_{k+1}=F\big(x_k, \pi(x_k), \epsilon_k\big), \hspace{10 pt} \text{with~}\epsilon_k  \hspace{10 pt} i.i.d, \footnote{For the sake of brevity, from now on we will exclude the reference to $\pi$ in the state update equations as a state-dependent policy implies there exists some function $G$ such that $F\big(x_k, \pi(x_k), \epsilon_k\big)=G(x_k, \epsilon_k)$.}
\end{equation}
where $x_{k} \in \mathbb{R}^n$ for all $k \in \mathbb{N}$ and $F: \mathbb{R}^n \times \mathbb{R}^n \times \mathbb{R}^n \longrightarrow \mathbb{R}^n $.
We assume that the transition kernel converges to some stationary distribution $\mu_{\pi}$ under Wasserstein metric $\mathcal{W}_{d}$ equipped with some distance function $d$. 

If we have access to empirical averages of some unbounded reward function $r(x)$, in this work, we explore the following questions:

\begin{itemize}
    \item \emph{Concentration from simulating a single trajectory}: 
    When, how and why can we provide something similar to following exponential concentrations
\begin{align}
\label{eq:expcon}  \mu^{N} \Bigg[\bigg|\frac{1}{N} \sum_{i=1}^{N} r(x_i) - <r>_{\mu_{\pi}} \bigg| > 
\epsilon \Bigg]  \leq 2 \exp \bigg(-\frac{N \epsilon^2}{K_{sys}(r)}\bigg),
\end{align}
where $r$ can be some unbounded function, in a control-theoretic or RL framework 
(e.g., $r(x):=\|x\|$), and $K_{sys}(r)$ is a constant dependent on system properties and 'smoothness' of $r$ (related to Lipschitz constant)?

    \item When explicit knowledge of the stationary distribution and system dynamics is not available, what are easy-to-verify sufficient conditions/functional inequalities 
    to derive concentration from a single trajectory? 
    
    \item In which ways stability of the considered dynamical system affects the concentration? and
    
    \item Is correlation between samples dependent on system  stability? 
\end{itemize}

\subsection{Contribution and Main Results}
Accordingly, the main contributions of this paper are as follows:
\begin{itemize}
    \item This paper's fundamental contribution is 
    a novel point of view towards getting exponential concentration inequalities for a random dynamical system, in an extremely tractable manner. To achieve this, we connect ideas and techniques \emph{from particle methods, functional inequalities and convergence of Markov chains in Wasserstein distance}.  
    
    \item  \emph{For stable LDS, we 
    show how to obtain exponential concentration of the form \eqref{eq:expcon} from single trajectory.}  
    The idea is that we can find a common metric on the space of probability measures, such that the transition kernels are contractive and uniformly satisfy transportation-entropy~inequality.
    
    \item Leveraging upon \emph{weighted T-E inequalities} developed by~\cite{bolley2005weighted}, we introduce an \emph{exponential-type Lyapunov condition for Markov chains}; if the Lyapunov condition holds, the stationary distribution satisfies the transport-entropy inequality. Consequently, empirical averages of the test function, evaluated on i.i.d samples from the stationary distribution, concentrate sharply around their mean.
    
    \item In case of non-linear random dynamical systems, such as Harris ergodic Markov chains (HEMCs), if exponential Lyapunov function exists, we show that one can simulate independent trajectories and, after a small burn-in period, average their rewards to obtain with high probability, a sharp estimate of the expected reward w.r.t the stationary distribution. 
\end{itemize}

\paragraph{Outline of paper.}
In Section \ref{sec:ldsconc}, we lay down a mathematical framework to obtain concentration for dependent random variables under the assumptions of uniform transport-entropy constants for a Markov transition kernel and Wasserstein contractivity. We conclude Section 2 with verification of the developed results on the problem of sample complexity in policy evaluation for average-reward based optimal control for LDS. Section  \ref{sec:sldsconc}  focuses on concentration for HEMCs that are not necessarily convergent in Wasserstein metric with the trivial euclidean distance. We show that if exponential-type Lyapunov function exists, empirical averages of test function evaluated on i.i.d samples from stationary distribution of Harris chain are sharply concentrated.
Finally, this phenomenon is verified on an example of an SLDS.

\section{Extending Concentration to Dependent Random Variables via Tensorization}
\label{sec:ldsconc}
\subsection{Preliminaries}
Before we introduce the mathematical framework to derive concentration for dependent random variables, we introduce the following results utilized later in this work. 

\begin{definition}
Consider metric space $(\mathcal{X},d)$ and reference  probability measure $\mu \in P(\mathcal{X})$. Then we say that $\mu$ satisfies $\mathcal{T}_{1}^d (C)$ or to be concise $\mu \in  \mathcal{T}_{1}^d (C)$ for some $C>0$ if for 
all $\nu \in P(\mathcal{X})$ it holds~that 
\begin{equation}
    \label{eq:t1}
    \mathcal{W}_d (\mu, \nu) \leq \sqrt{2 C Ent(\nu||\mu)}.
\end{equation}
\end{definition}

\begin{lemma}[\cite{bobkov1999exponential}]
\label{lm:b-g}
$\mu$ satisfies $\mathcal{T}_{1}^d (C)$ if and only if for 
any Lipschitz function $f$ with $<f>_{\mu}:= \mathbb{E}_{\mu} f$, it holds that
\begin{flalign}
    \label{eq:bg} & \int e^{\lambda(f- <f>_{\mu})} d\mu \leq \exp(\frac{\lambda^2}{2}C \|f\|_{L(d)} ^2), \hspace{15pt} \text{where~~~~}  \|f\|_{L(d)}:= \sup_{x \neq y} \frac{|f(x)-f(y)|}{d(x,y)}.
\end{flalign}
\end{lemma}

\begin{remark}
\label{rm: iid}
(\ref{eq:bg}) along with the Markov inequality implies that if we sample $x$ from $\mu \in \mathcal{T}_{1}^d(C)$, then
\begin{equation}
    \label{eq:iid} \mathbb{P} \Bigg[ \bigg| r(x) -<r>_{\mu}\bigg| > \epsilon \Bigg] \leq 2\exp\bigg(-\frac{ \epsilon ^2}{2C \|r\|_{L(d)}^2}\bigg).
\end{equation}
\end{remark}
Now, consider a Markov chain  $ x^N:=(x_i)_{i=1}^{N}$ with distribution $\mu^{N} \in P$ $(\mathcal{X}^N)$ and $P^{m}(x,\mathcal{B}):= \mathbb{P}(x_m \in \mathcal{B}| x_0=x)$, for all Borel subsets $\mathcal{B}$ of $\mathcal{X}$. We can extend the metric $d$ to $\mathcal{X}^N$ as
\begin{equation}
    \label{eq:tensor}
    d_{(N)}(x^N,y^N)= \sum_{i=1}^{N} d(x_i,y_i).
\end{equation}
If $\mu^{N} \in \mathcal{T}_{1}^{d_{(N)}} \big( \mathcal{O} (N) \big)$ and $r$ is one Lipschitz, i.e., $\|r\| _{L(d)} \leq 1$, then $\Phi(x^{N}):= \frac{1}{N} \sum_{i=1}^{N} r(x_i) $ satisfies $\|\Phi\|_{L(d_{(N)})} \leq \frac{1}{N}$; plugging these results into~\eqref{eq:bg}, we obtain that
\begin{align}
     & \label{eq:MCC}
    \mu^{N} \Bigg[\bigg|\frac{1}{N} \sum_{i=1}^{N} r(x_i) -\mathbb{E}\Bigg(\frac{1}{N} \sum_{i=1}^{N} r(x_i) \Bigg) \bigg| > \epsilon \Bigg]    \leq 2\exp\bigg(-\frac{N \epsilon ^2}{2C}\bigg).
\end{align}

\subsection{Contractivity and Uniform Transport Constants }

As one would wonder from  \eqref{eq:tensor}, when does the T-E  for process level law of Markov chain, increases at worse linearly with dimension (in sample term)?
Sufficient conditions (see e.g., \cite{djellout2004transportation,bolley2005weighted})~are
 \begin{align}
    \label{eq:unfrmtran}
    (i)&\hspace{10pt} P(x,\cdot) \in T_1^{d}(C), \hspace{72pt} \text{for all}~{x \in \mathcal{X}}, \text{and some}~{C>0},  \\
    \label{eq:wascon}
    (ii)&\hspace{10pt}  \mathcal{W}_{d}(P(x,\cdot),P(y,\cdot))  \leq \hat{\lambda}d(x,y),\hspace{8pt} \text{for all }~{(x,y) \in \mathcal{X}^{2}} \text{and some}~{\hat{\lambda} \in [0,1)}.
    \end{align}

 Property \eqref{eq:unfrmtran} is often referred to as existence of a uniform transportation constant and  \eqref{eq:wascon} represents contractivity of the Markov Chain in the Wasserstein metric / \emph{spectral gap in the Wasserstein sense}. 
 Now, 
 the following result holds.
 
 \begin{lemma}
\label{cl:tranN} If \eqref{eq:unfrmtran} and  \eqref{eq:wascon} hold,  process level distribution of samples from a Markov chain $(x_1,\ldots,x_N)$, which we will 
 denote as $Law(x_1,\ldots,x_N)$, 
 denoted by $\mathcal{\mu}^{N}$ satisfies  $T_{1}^{d_{(N)}} \bigg( \frac{CN}{(1-\hat{\lambda})^2}\bigg)$, for all $N \in \mathbb{N}$.
 \end{lemma}
 
\begin{proof} 
See Theorem  2.5 of \cite{djellout2004transportation} for a detailed proof.
\end{proof}

\subsection{Sharp Deviation Inequalities for Average-Reward Based Optimal Control}

In optimal control and 
RL literature, it is often the case that under the action of some state dependent policy, the resulting closed-loop dynamical system under consideration, $(x_i)_{i=1}^{N}$, mixes to some stationary distribution $\mu_{\pi} \in P(\mathcal{X})$. When exact system
parameters and state are unknown but time-averages of the reward function $r(x):=\|x\|$, although unbounded, are available, 
sharp deviation bounds of $\frac{1}{N} \sum_{i=1}^{N} r(x_i)$ from $<r>_{\mu_{\pi}}$ are of utmost importance from the sample complexity point of view.
Sufficient conditions for concentration of empirical averages of $r(x):=\|x\|$ or any Lipschitz function essentially boils down to showing that the process level law of the Markov chain $\mu^{N} \in \mathcal{T}_{1}^{d_{(N)}} \big( \mathcal{O} (N) \big)$.

\begin{theorem}
\label{thm:structure} Assume that under a metric $d$, a random dynamical system uniformly satisfies T-E inequality with constant $C>0$ and is contractive in Wasserstein sense with constant $\gamma \in (0,1)$. Then, for any 1 Lipschitz function $r$, the following deviation inequality for empirical averages of $r$~holds
\begin{flalign}
\label{eq:thmpfstruct} \mu^{N} \Bigg[\bigg|\frac{1}{N} \sum_{i=1}^{N} r(x_i) - & <r>_{\mu_{\pi}} \bigg|  > \frac{\mathcal{W}_{d} (P(x. \cdot), \mu_{\pi})}{N( 1-\gamma)}+ \epsilon \Bigg]  \leq  2\exp\bigg(-\frac{N \epsilon ^2 (1-\gamma)^2}{2C}\bigg).
\end{flalign}
\end{theorem}
\begin{proof}
Since the theorem assumptions satisfy the claim in \eqref{cl:tranN}, it holds that $\mu^{N} \in \mathcal{T}_{1}^{d_{(N)}} \big(\frac{CN}{(1-\gamma)^2}\big)$. Let $(y_i)_{i=1}^{N}$ be i.i.d samples from $\mu_{\pi}$ and assume that the Markov chain starts deterministicly with $x_1=x$. Then, we have for all $\epsilon >0$ it holds
\begin{align} 
\nonumber & \mu^{N} \Bigg[\bigg|\frac{1}{N} \sum_{i=1}^{N} r(x_i) - <r>_{\mu_{\pi}} \bigg| > \frac{\mathcal{W}_{d} (P(x, \cdot), \mu_{\pi})}{N( 1-\gamma)}+ \epsilon \Bigg] \\ & \nonumber  \leq \mu^{N}  \Bigg[\bigg|\frac{1}{N} \sum_{i=1}^{N} r(x_i) -\mathbb{E}\Bigg(\frac{1}{N} \sum_{i=1}^{N} r(x_i) \Bigg) \bigg|+\bigg| \mathbb{E}\Bigg(\frac{1}{N} \sum_{i=1}^{N} r(y_i)\Bigg) -\mathbb{E}\Bigg(\frac{1}{N} \sum_{i=1}^{N} r(x_i) \Bigg) \bigg|  >  \frac{\mathcal{W}_{d} (P(x. \cdot), \mu_{\pi})}{N( 1-\gamma)}+ \epsilon \Bigg] \\ & \label{eq:csch} \leq \mu^{N} \Bigg[\bigg|\frac{1}{N} \sum_{i=1}^{N} r(x_i) -\mathbb{E}\Bigg(\frac{1}{N} \sum_{i=1}^{N} r(x_i) \Bigg) \bigg| > \epsilon \bigg] \\ & \label{eq:crux} \leq   2\exp\bigg(-\frac{N \epsilon ^2 (1-\gamma)^2}{2C}\bigg),
\end{align}
where (\ref{eq:csch}) follows from contractive dynamics in Wasserstein distance.
\end{proof}

\vspace{-10pt}
\begin{remark}
The structure of the aforementioned concentration inequality is inspired by the work in~\cite{malrieu2001logarithmic} on bounding deviations of empirical measure formed by interacting particle systems from their infinite particle limit (Mckean-Vlasov diffusion); see e.g., problem section in \cite{villani2003topics} for an involved discussion on this matter. However, there is a subtle difference in our approach, as we 
work with an $\ell^{2}$ inspired metric on $\mathcal{X}^{N}$ as in \cite{malrieu2001logarithmic}, i.e., 
\vspace{-4pt}
\begin{equation}
    \label{eq:l2ten} 
    d_{(N)} ^2 (x^N,y^N):= \sqrt{\sum\nolimits_{i=1}^{N} d^2(x_i,y_i)}.
\end{equation}

Then, $\Phi(x^{N}):= \frac{1}{N} \sum_{i=1}^{N} r(x_i) $ w.r.t $\ell^2$ metric is only $\frac{1}{\sqrt{N}}$ Lipschitz  --- i.e., $\|\Phi\|_{L(d_{(N)} ^2)} \leq \frac{1}{\sqrt{N}}$ and any hope for concentration would require $\mu^{N} \in \mathcal{T}_{1}^{d_{(N)} ^2} \big( \mathcal{O} (1) \big)$ (i.e., \emph{dimension free concentration}), which is very difficult to check; this is feasible only in Markov diffusion processes with uniformly convex external potentials and symmetric interaction potentials, see e.g., \cite{malrieu2001logarithmic}.
\end{remark}

\paragraph{Decay of correlation.} By combining conditions from \eqref{eq:unfrmtran} and \eqref{eq:wascon}, with Taylor's expansion for small $\lambda$ (terms of order up to $\lambda^2$) appearing on both sides in Bobkov-Gotze dual form \eqref{eq:bg}, 
for all $x \in \mathcal{X}$ it holds that 
\vspace{-10pt}
\begin{align}
    \label{eq:deccor} |Cov_{P_x} [f(x_n),f(x_{n+k})]| \leq \frac{\hat{\lambda}^{k}}{1-\hat{\lambda}^2} C \|f\|_{L(d)} ^{2}. 
\end{align}

\subsection{Sharp Concentration via Single Trajectory of Stable Linear Dynamical Systems}

Consider a discrete-time linear dynamical system (LDS) of the form
\begin{equation}
    \label{eq:LDS} y_{t+1}= Ay_t+ \epsilon_{t}, \hspace{10pt} \|A\|_{2}= \hat{\lambda}<1 \hspace{10pt} \text{and i.i.d }~ \epsilon_{t} \thicksim \mathcal{N}(0,\mathcal{I}_n).
\end{equation}
\emph{Notion of global stability in control theory is related to $\rho(A)<1$; however, Gelfand's formula implies for any $\rho \in (\rho(A),1)$, there exists $ 0<M(\rho)$ such that $\|A^n\| \leq M(\rho) \rho^n$ } for all $n \in \mathbb{N}$ and concentration inequalities are similar (although it will require taking a similarity transformation as in Proposition 4.2 of \cite{blower2005concentration}, which we avoid here due to limitation of space)

With the trivial euclidean metric $d(x,y):=\|x-y\|$, the transition kernel from \eqref{eq:LDS} satisfies $P(x,\cdot) \in T_1^{d}(1), \hspace{5pt} \forall x \in \mathcal{X}$   \cite{talagrand1996transportation} and $\mathcal{W}_{d}^{2}(P(x,\cdot),P(y,\cdot)) =\|Ax-Ay \| \leq \hat{\lambda}d(x,y)$ (see e.g., \cite{givens1984class}). 
Now, an application of Jensens' inequality reveals $\mathcal{W}_{d}(P(x,\cdot),P(y,\cdot)) \leq \mathcal{W}_{d}^{2}(P(x,\cdot),P(y,\cdot))$ and contractivity follows. As conditions (\ref{eq:unfrmtran}) and (\ref{eq:wascon}) are satisfied, we can use coupling technique inspired by \cite{marton2004measure} to prove that $\mathbb{P}^{N}:=$Law$(y_1,\ldots, y_N)$ of the LDS 
satisfies $T_{1}^{d_{(N)}}(\frac{N}{(1-\hat{\lambda})^{2}})$; and if $\nu_{\pi}$ is the invariant measure corresponding to \eqref{eq:LDS}, 
we have
\begin{flalign}
    & \nonumber \mathbb{P}^{N} \Bigg[\bigg|\frac{1}{N} \sum_{i=1}^{N} \|y_i\| - <\|y\|>_{\nu_{\pi}} \bigg| > \frac{\mathcal{W}_{d} (P(x. \cdot), \nu_{\pi})}{N( 1-\hat{\lambda})}+ \epsilon \Bigg] \leq \\  & \label{eq:ldsconc}  \leq \mathbb{P}^{N} \Bigg[\bigg|\frac{1}{N} \sum_{i=1}^{N} \|y_i\| -\mathbb{E}\Bigg(\frac{1}{N} \sum_{i=1}^{N} \|y_i\| \Bigg) \bigg| > \epsilon \Bigg]   \leq  2\exp\bigg(-\frac{N \epsilon ^2 (1-\hat{\lambda})^2}{2}\bigg).
\end{flalign}

 \section{Concentration for Nonlinear Random Dynamical Systems: The Case of  Harris Ergodic Markov Chains }
 \label{sec:sldsconc}

 The case of nonlinear random dynamical systems suffers from a lack of uniform transport-entropy constant related to a contractive metric. However, if we can show that the invariant measure satisfies transport-entropy inequality and Markov chain converges exponentially fast to its stationary distribution: it is plausible to run independent simulations of Markov chain and sample the averages after some burn-in period. 

Since, the results developed in this paper are aimed at facilitating the RL and controls community. In the absence of exact dynamics, we develop easily verifiable/realistic conditions that ensures exponential integrability of invariant measure. This brings us to the weighted transportaion-inequalities introduced in 
\cite{bolley2005weighted} and plays an integral role when studying concentration phenomenon for nonlinear random dynamical systems. Their work allows for adding different weigths to the underlying distance function, precisely said:
\begin{lemma}
\label{lm:gaussian}
Let $\phi$ be a non-negative integrable function, such that $\int e^{\phi(x)^{2}} \mu_{\pi}(dx) <\infty$- we get an upper bound on weighted total variation distance

\begin{equation}
    \label{eq:wtd} 
    \| \phi (\mu_{\pi}-\nu)\|_{TV} \leq \sqrt{2} \bigg(1+\log \int e^{\phi(x)^{2}} \mu_{\pi}(dx)\bigg)^{\frac{1}{2}} \sqrt{Ent(\nu || \mu_{\pi})}. 
\end{equation}
\end{lemma}

A remarkable advantage of this formulation is an Lyapunov condition for underlying Markov chain so as to ensure that at least its' stationary measure satisfies T-E inequality.
\paragraph{Exponential Lyapunov function.} Inspired by the assumption made in (particular case 14 of \cite{bolley2005weighted}), we proposed an exponential Lyapunov condition:
\begin{enumerate}
    \item \label{eq:C1explp}  There exists $\hat{\alpha}>0$, $\beta>0$ and $C>0$ such that $\beta< \hat{\alpha}$ and: 
\begin{flalign}
    & \nonumber \int e^{\hat{\alpha} \|y\|^{2}}P(x,dy) \leq Ce^{\beta\|x\|^{2}}, \hspace{10pt} \text{for all }~x \in \mathcal{X}.
\end{flalign}
\end{enumerate}
\begin{theorem}
\label{thm:expneat} If the exponential Lyapunov condition is satisfied, define $W_{\hat{\alpha}} (x):= e^{\hat{\alpha} \|x\| ^{2}}$, then $n-$ th step transition kernel $P^{n}(x, \cdot)$ satisfy the following transport entropy inequality:
\begin{equation}
\label{eq:TE(n-step)}
    \mathcal{W}_d (P^{n}(x, \cdot), \nu) \leq \sqrt{2} \bigg( \frac{1+\log P^{n} W_{\hat{\alpha}} (x)}{\hat{\alpha}} \bigg)^{\frac{1}{2}} \sqrt{ Ent(\nu||P^{n}(x, \cdot))}
\end{equation}
Moreover, if an ergodic invariant measure $\mu_{\pi}$ exists: then exists also a finite positive constant $L_{\hat{\alpha}, \beta, C }$ such that:
\begin{equation}
     \mathcal{W}_d (\mu_{\pi}, \nu) \leq \sqrt{2 L_{\hat{\alpha}, \beta, C } Ent(\nu||\mu_{\pi})}
\end{equation}
\end{theorem}

\begin{proof}
   We will only prove the result for the invariant measure as the result for $n-$ th step will follow the same argument. Since the condition in hypothesis can also be written as: 
\begin{flalign}
\int e^{\hat{\alpha} \|y\|^{2}}P(x,dy) \leq Ce^{(\beta-\hat{\alpha})\|x\|^{2}}e^{\hat{\alpha}\|x\|^{2}}
\end{flalign}
As $(\beta-\hat{\alpha}) < 0$, we can find $\eta_{\hat{\alpha}} \in (0,1)$ and $\hat{C}_{\hat{\alpha}}< \infty$ such that $W_{\hat{\alpha}} (x)$ satisfy:
\begin{flalign}
    & \label{eq:lyapexpp2} PW_{\hat{\alpha}} (x) \leq \eta_{\hat{\alpha}} W_{\hat{\alpha}} (x)+ \hat{C}_{\hat{\alpha}}, \hspace{5pt}  \text{and consequently via recursion} \hspace{5pt} \int e^{\hat{\alpha}\|x\|^2} \mu_{\pi}(dx) \leq \frac{\hat{C}_{\hat{\alpha}}}{1-\eta_{\hat{\alpha}}}.
\end{flalign}   
and by defining $\phi(x)= \sqrt{\hat{\alpha}}\|x\|$, upper bound on
 weighted total variation from Lemma \ref{lm:gaussian} implies that :
 \begin{align}
     \|\phi(\mu_{\pi}- \nu)\|_{TV} \leq \sqrt{2} \Big(1+ \log \bigg[  \frac{\hat{C}_{\hat{\alpha}}}{1-\eta_{\hat{\alpha}}} \bigg]   \Big)^{\frac{1}{2}} \sqrt{Ent(\nu || \mu_{\pi})}.
 \end{align}
 Since, Wasserstein distance is upper bounded by weighted total-variation with weight $\|x\|$, after scaling we conclude that $\mu_{\pi} \in  \mathcal{T}_{1}^d \Bigg(\frac{1+ \log \bigg[  \frac{\hat{C}_{\hat{\alpha}}}{1-\eta_{\hat{\alpha}}} \bigg]}{\hat{\alpha}}   \Bigg)$ and the result for $n$-th step transition kernel follows via same argument.
\end{proof}

\paragraph{Harris chains.} As the notion of running multiple independent trajectories is plausible when burn-in period is negligible : $n-$th step transition kernel of Markov chain converges exponentially fast to an invariant measure in Wasserstein metric, this brings us to Harris ergodic Markov chains that by definition satisfy following conditions:

\emph{Lyapunov condition with geometric drift:} There exists a Lyapunov function $V:\mathcal{X} \rightarrow [0,+\infty)$, which satisfies:
\vspace{-4pt}
 \begin{equation}
     \label{eq;lyapgeom} PV(x) \leq \hat{\gamma} V(x)+K, \hspace{8pt} \text{for some }~{ \hat{\gamma} \in (0,1)} \hspace{3pt}  \text{,}~{K < \infty} \hspace{5pt} and
\end{equation}

\emph{minorization condition:} 
A sufficiently large level set of $V$ (ironically it is called `small set'), satisfies the minorization condition: i.e., there exists a set $\mathcal{S}:=\{x \in \mathcal{X}: V(x) \leq R \}$  for some $R> \frac{2K}{1- \hat{\gamma}}$, $\beta \in (0,1)$ and $\hat{\nu} \in \mathcal{P(X)} $ such that:
\vspace{-2pt}
\begin{equation}
    \label{eq:minorization} \mathcal{P}(x, \cdot) \geq \beta \chi_{\mathcal{S}} (x) \hat{\nu}(\cdot).
\end{equation}
Under these conditions it was shown by \cite{hairer2011yet} that for some $\beta^*>0$HEMC is contractive in Wasserstein metric $\mathcal{W}_{d}$ with distance function: $ d(x,y):=(2+\beta^*V(x)+\beta^*V(y)) \chi_{x \neq y},$ . A unique ergodic invariant measure $\mu_{\pi}$ exists and for some finite $C$ and $\kappa \in (0,1)$
\begin{equation}
    \mathcal{W}_{d}\big(P^{n}(x, \cdot), \mu_{\pi}) \leq C \kappa^{n} \mathcal{W}_{d} (P(x,\cdot), \mu_{\pi}).
\end{equation}
\subsection{Application to Concentration for SLDSs}


\paragraph{Model specifications.} We consider a discrete-time 
SLDS of the form 
\vspace{-2pt}
\begin{equation}
\label{eq:sls}
x_{t+1}=\sum\nolimits_{j=1}^{M} (A_j x_t+ w_{t}^{j})\chi_{\mathcal{M}_j} (x_t).
\end{equation}
Here, $x_t \in \mathbb{R}^n$ denote the system's state  and $A_j \in \mathbb{R}^{n \times n}$ for $j=1,...,M$ capture system dynamics in each of the $M$ Borel measurable regions that decompose the state-space and are pairwise disjoint satisfying $\bigcup_{j=1}^{M} \mathcal{M}_j = \mathbb{R}^{n}$.
In addition, for a fixed region $j$, noise vectors $w_{t}^{j} $ are i.i.d, and satisfy $w_{t}^{j} \thicksim \mathcal{N}(0,{I}_{n})$ 
and 
$Cov(w_{t}^j,w_{s}^k)=0$, 
for all $t,s\geq 0$ and $j \neq k \in \{1,2,...,M\}$.  

\begin{lemma}
\label{lm:prvs}  
Assume that there exists $ \varrho< \infty$ such that for all $l \in \mathcal{K}_{bdd}:=$ 
$\{ k ~|~ ( 1\leq k \leq M)  \text{ such that } \mathcal{M}_k \subsetneq \mathcal{B}^n_{\varrho}\}$,it holds that $\| A_l \|_{2} \leq L< \infty$ and $\forall j \in (\mathcal{K}_{bdd})^\complement \hspace{2pt}$,  $\|A_j\|_2 \leq \gamma <1 $. Then,~the system \eqref{eq:sls} mixes geometrically to a unique ergodic invariant distribution~$\mu_{\pi}$.
\end{lemma}
\begin{proof}
Consider function $V(x)=\|x\|_{2}$. From~\eqref{eq:sls}, we have $P(x,\mathcal{A})= \sum_{j=1}^{M}P_j(x,\mathcal{A})\chi_ {\mathcal{M}_j}(x)$, where $P_j(x,\cdot) \thicksim \mathcal{N}(A_j x, {I}_n)$. 
Assuming the initial state  $x_0:=x \in \mathcal{M}_{k}$ for some $ k \in \mathcal{K}_{bdd}$, then:
\begin{align}
 \label{eq:t9} PV^2(x)=  \mathbb{E}_{y \thicksim \mathcal{N}(A_kx, I_n)} \|y\|_{2} ^{2}=  \mathbb{E}_{z \thicksim \mathcal{N}(0,I_n)} \|z\|_2 ^{2} + \|A_kx\|_{2} ^{2}   \leq  (n+L^2\varrho^{2}).
\end{align}
However, if the initial state is $x_0:=x \in \mathcal{M}_{j}$ such that $j \in  (\mathcal{K}_{bdd})^\complement $, then:
\begin{align}
\label{eq:t10}
PV^2(x)= \mathbb{E}_{y \thicksim \mathcal{N}(A_jx, I_n)} \|y\|_{2} ^{2}= \mathbb{E}_{z \thicksim \mathcal{N}(0,I_n)} \|z\|_2 ^{2} + \|A_jx\|_{2} ^{2}   \leq  n+\gamma^2 \|x\|_{2}^{2}=(n+\gamma^2 V^2(x)).
\end{align}
Therefore, starting from any initial condition in $\mathbb{R}^n$, from~\eqref{eq:t9} and~\eqref{eq:t10} it holds that $PV^2(x) \leq \gamma^2 V^2(x) + 
(n+L^2\varrho^{2})$ and a trivial application of Jensen inequality reveals
\begin{equation}
\label{eq:t12}
PV(x) \leq \gamma V(x) + 
\underbrace
{\sqrt{n+L^2\varrho^{2}}}_{K}.
\end{equation}  
Minorization condition can be verified from \cite{naeem2020learning} and the result follows. 
\end{proof}

\begin{theorem}
\label{cl:BE}
 For any $\hat{\alpha} \in (0, \frac{1-\gamma^2}{2})$, we have $\int e^{\hat{\alpha}\|x\|^2} \mu_{\pi}(dx) < \infty$ and consequently there exists a finite positive constant $L_{\gamma}$ such that the invariant measure of SLDS, $\mu_{\pi} \in  \mathcal{T}_{1}^d \big( L_{\gamma} \big)$.
\end{theorem}
\begin{proof}
 An application of Stein's lemma on transition kernel of \eqref{eq:sls}, reveal:
\begin{flalign}
     & \nonumber \int e^{\alpha \|y\|^{2}}P(x,dy)= \frac{1}{(1-2\alpha)^{\frac{n}{2}}} e^{\|A_jx\|^2 (\alpha+ \frac{2 \alpha ^2}{(1-2\alpha)})} 
     = \frac{1}{(1-2\alpha)^{\frac{n}{2}}} e^{\|A_jx\|^2 \frac{ \alpha }{(1-2\alpha)}}, \hspace{10pt} \alpha< \frac{1}{2}, \hspace{2pt} x \in \mathcal{M}_{j}.
\end{flalign}
A simple linear algebra exercise reveals existence of a $\beta< \hat{\alpha}$ and $C$ when $\hat{\alpha} \in (0, \frac{1-\gamma^2}{2})$ as mentioned in Theorem \ref{thm:expneat} and conclusion follows.  
\end{proof}

\subsection{Gaussian Tail Inequality for Stationary Distribution of SLDS/Harris Chain with Exponential Lyapunov Function and its Consequences for Sampling}
Assume that we have access to sampling $(y_i)_{i=1}^{N}$ i.i.d from $\mu_{\pi}$. Since, $r(x):=\|x\|$ is $1-$ Lipschitz w.r.t $d(x,y):= \|x-y\|_{2}$ and $\mu_{\pi} \in  \mathcal{T}_{1}^d \Bigg(\frac{1+ \log \bigg[  \frac{\hat{C}_{\hat{\alpha}}}{1-\eta_{\hat{\alpha}}} \bigg]}{\hat{\alpha}}   \Bigg)$, we have that
\vspace{-10pt}
\begin{align}
     \label{eq:iidinv}  \mathbb{P} \Bigg[ \bigg|\frac{1}{N} \sum_{i=1}^{N} r(y_i) -<r>_{\mu_{\pi}}\bigg| > \epsilon \Bigg] \leq 2\exp\Bigg(-\frac{N \epsilon ^2 \hat{\alpha} }{2+ 2\log \bigg[  \frac{\hat{C}_{\hat{\alpha}}}{1-\eta_{\hat{\alpha}}} \bigg]}\Bigg). 
\end{align}
As any valid $\hat{\alpha}$ can be written down in the form of $\frac{1-\gamma^2}{2n}$ for $n>1$, comparing with Linear Gaussian case \eqref{eq:ldsconc} it is reassuring to see how deviations for SLDS and LDS have similar dependence in terms of the norm of stable system matrix.

\begin{remark}
Although, we tried our best to show concentration for process level law of HEMCs under exponential-type Lyapunov condition, via weigthed T-E inequality but it got intractable due to non-uniform transport constants. 
\end{remark}

\section{Conclusion and discussion }
We have provided with a general framework for getting concentration inequalities for random dynamaical systems, with respect to empirical averages of unbounded test functions; validated our analysis on the example of LDS and SLDS. Summarizing few key observations and open problems:
 \emph{Exponential-type Lyapunov functions are sufficient for having exponential concentration inequalities}  (although, we might have to run multiple independent trajectories). \emph{Stability is necessary not sufficient:} As it should be evident from the example of HEMCs, under existence of exponential-type Lyapunov functions we can ensure concentration only by running multiple independent trajectories. Concentration via single trajectory requires some notion of regularity for transition kernels as well: one of them is uniform transportation constant and we intend on exploring more such regularity conditions in our future work. An interesting open problem, \emph{can we relax uniform T-E condition and still get exponential inequalities from a single trajectory?} We suspect it might be the case when samples of Markov chain interact symmetrically; convincing examples include, Kac's interacting particles model exhibiting propagation of chaos and ubiquity of assumption on reversibility of Markov chain w.r.t its' stationary maeasure in discrete and continuous time Large Deviation Principle, see e.g, \cite{wang2020transport} and \cite{guillin2009transportation}.

\bibliography{l4dc2023-sample}

\end{document}